\documentclass{article} 
\usepackage{iclr2027_conference,times}

\usepackage{amsmath,amsfonts,bm}

\def\eqref#1{equation~\ref{#1}}

\def\1{\bm{1}}

\DeclareMathAlphabet{\mathsfit}{\encodingdefault}{\sfdefault}{m}{sl}
\SetMathAlphabet{\mathsfit}{bold}{\encodingdefault}{\sfdefault}{bx}{n}

\usepackage{hyperref}
\usepackage{url}

\usepackage{amsmath,amsthm}
\usepackage{enumitem}                                  
\usepackage{algorithm}
\usepackage{algorithmic}
\usepackage{graphicx}
\usepackage{subcaption}
\usepackage{wrapfig}
\usepackage{multirow}
\usepackage{booktabs}

\newtheorem{theorem}{Theorem}
\newtheorem{assumption}{Assumption}

\newcommand{\pengxin}[1]{\textcolor{black}{#1}}

\title{Fed-GRPO: Reward-Signal-Driven Federated Group Relative Policy Optimization}

\author{Pengxin Guo$^{1}$ Shuang Zeng$^{1}$ Zonggen Li$^{1}$ Weiying Zheng$^{1}$ Mengting Liu$^{2}$ Liangqiong Qu$^{1}$\thanks{Corresponding author.}~~
\\
$^{1}$ The University of Hong Kong \quad $^{2}$ Sun Yat-sen University \\
\texttt{\{guopx,zengsh9,zonggen.li,weiyingzheng\}@connect.hku.hk,}\\
\texttt{liumt55@mail.sysu.edu.cn, liangqqu@hku.hk} \\
}

\iclrfinalcopy 
\begin{document}

\maketitle

\begin{abstract}
Large Language Models (LLMs) have shown strong reasoning capabilities when fine-tuned with reinforcement learning (RL), particularly through Group Relative Policy Optimization (GRPO). However, existing GRPO methods assume centralized access to training data, which may not hold in practice due to privacy or regulatory constraints. To this end, we propose Fed-GRPO, a federated GRPO training framework that addresses these privacy constraints by enabling collaborative reasoning training without sharing raw data, which leverages the reward statistics naturally produced during GRPO training as zero-cost signals to guide aggregation, local training, and communication. Fed-GRPO contains three reward-signal-driven mechanisms: (i) \emph{signal-weighted aggregation} that weights clients by their reward standard deviation, prioritizing clients with stronger learning signals; (ii) \emph{global reward calibration} that re-weights per-prompt objectives based on the local-global reward gap, steering each client toward its relative weaknesses; and (iii) \emph{adaptive sparse communication} that allocates bandwidth based on the informativeness of each client's update. Extensive experiments on mathematical reasoning tasks demonstrate that Fed-GRPO achieves the best performance among all federated methods, clearly outperforms FedAvg and approaches centralized training performance, while losslessly reducing communication by $32\times$ and supporting up to $621\times$ compression under tight bandwidth budgets with only graceful accuracy degradation. Our code is available at \href{https://github.com/HKU-HealthAI/Fed-GRPO}{https://github.com/HKU-HealthAI/Fed-GRPO}.
\end{abstract}

\section{Introduction}
\label{sec:intro}

Large Language Models (LLMs) have shown strong reasoning capabilities when fine-tuned with reinforcement learning (RL), in particular through Group Relative Policy Optimization (GRPO)~\cite{shao2024deepseekmath}, which computes advantages from group-relative reward comparisons without a separate critic model. Building on this, DeepSeek-R1~\cite{guo2025deepseek} demonstrates that GRPO with simple rule-based rewards can produce strong reasoning behavior. Since then, a growing line of work has improved various aspects of GRPO-style training, including training stability~\cite{yu2025dapo}, advantage estimation granularity~\cite{liu2025understanding, sun2025ktae, yu2026unveiling}, and multi-reward optimization~\cite{liu2026gdpo}. However, these methods almost universally assume centralized access to training data. This assumption is increasingly unrealistic in high-value reasoning domains, where data is distributed across institutions, platforms, or jurisdictions. For example, hospitals, financial firms, educational platforms, and enterprise coding environments may each possess valuable reasoning data, but cannot pool it due to privacy, ownership, or regulatory constraints. 

Federated Learning (FL)~\cite{mcmahan2017communication,zeng2024tackling,zheng2025fedvlmbench,guo2025exploring}, which enables collaborative model training across institutions without sharing raw data, offers a natural solution. However, federated GRPO is not as simple as running GRPO inside a standard FL loop. The core difficulty is that, in GRPO, a client's usefulness is governed by its \emph{reward signal}. This makes federated GRPO fundamentally different from conventional supervised FL. In conventional supervised FL, client importance is usually determined by sample counts or controlled through parameter-, gradient-, or representation-space drift correction~\cite{mcmahan2017communication,li2020federated,karimireddy2020scaffold,li2021model}. In GRPO, however, a large client may be nearly saturated and provide little learning signal, whereas a smaller client with high reward dispersion may produce much more informative policy gradients.
Thus, the server must answer a different question: \emph{which clients are currently useful for improving reasoning?} 

Notably, a closer examination shows that this reward-dependent notion of usefulness extends far beyond simple client aggregation; it fundamentally alters the entire federated training loop (see Figure~\ref{fig:motivation}(a)). 
First, \textbf{aggregation is reward-misaligned}. 
FedAvg~\cite{mcmahan2017communication} and its follow-up work~\cite{li2020federated,karimireddy2020scaffold,li2021model} weight clients by dataset size or uniformly, at most mitigating drift in parameter-, gradient-, or representation-space.
These methods do not capture the \emph{reward-space signal} in GRPO, making the resulting global update misaligned with the clients that are most useful for RL training. Second, \textbf{local objective is reward-miscalibrated}. 
Each client performs GRPO using only local reward statistics, and therefore lacks a global view of prompt difficulty. 
This can cause clients to under-train globally lagging prompts and over-train prompts already mastered locally. Third, \textbf{communication remains reward-agnostic}. Each federated round requires transmitting billion-parameter models or updates, creating a severe bandwidth bottleneck.
Existing compression methods~\cite{lin2018deep,guastella2025sparsyfed,sarfi2025communication} usually rely on parameter-space heuristics, which may over-compress reward-informative clients while wasting bandwidth on uninformative ones.

\begin{figure} [t]
    \centering
    \includegraphics[width=0.99\textwidth]{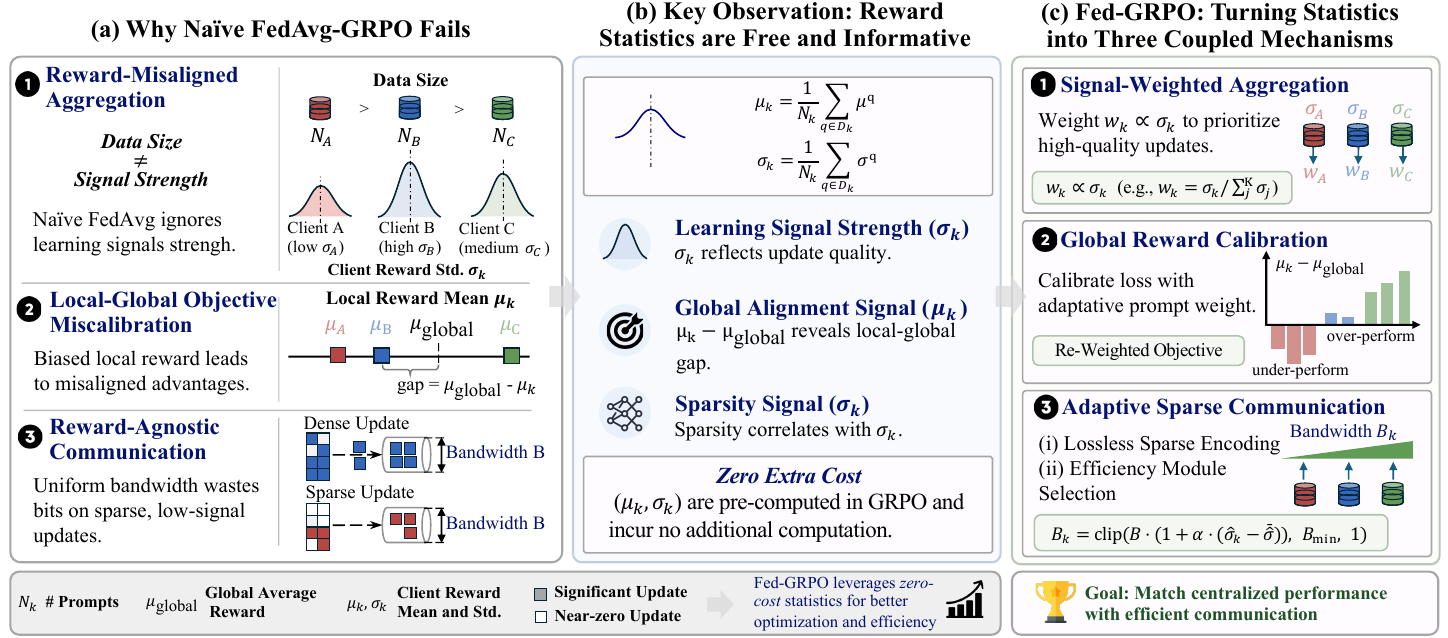}
    \caption{\textbf{Motivation of Fed-GRPO.} Naïve FedAvg-GRPO suffers from three coupled issues: (1) reward-misaligned aggregation, (2) local--global objective miscalibration, and (3) reward-agnostic communication. We observe that the reward statistics $(\mu_k, \sigma_k)$, computed by GRPO at \emph{zero extra cost}, are precisely the signals needed to address all three: $\sigma_k$ guides both signal-weighted aggregation and adaptive bandwidth allocation, while the local--global gap on $\mu_k$ calibrates per-prompt objectives.}
\label{fig:motivation}
\end{figure}

To investigate these challenges, we conduct a diagnostic experiment by training Qwen2.5-3B-Instruct~\cite{qwen2025qwen25} on the MATH dataset~\cite{hendrycks2021measuring} under standard FedAvg-GRPO with a non-IID partition (Section~\ref{sec:motivation}).  We observe three properties that connect directly back to the three challenges above:
(1)~per-client reward standard deviations $\sigma_k$ vary substantially across clients within the same round, suggesting that learning signal strength---not dataset size---should drive aggregation;
(2)~per-client reward means $\mu_k$ diverges significantly from the global average, revealing a persistent local-global gap that biases local advantage estimation and misallocates training effort; and
(3)~GRPO updates are naturally sparse, with over $96.2\%$ of parameters unchanged after local training, and the update density correlates with $\sigma_k$, opening the door to adaptive sparse communication that gives more bandwidth to clients with stronger learning signals.

These observations point to a simple but powerful principle: \emph{the reward statistics computed by GRPO contain precisely the information needed to guide federated aggregation, local objective calibration, and communication allocation}. Importantly, these statistics are available at essentially  \textbf{zero additional training cost}. 
Building on this insight, we propose \textbf{Fed-GRPO}, a federated GRPO training framework that turns these reward statistics into three coupled mechanisms: 
(i)~\textbf{signal-weighted aggregation} weights clients by $\sigma_k$ instead of dataset size, so that clients with stronger learning signals contribute more to the global update;
(ii)~\textbf{global reward calibration} re-weights per-prompt objectives based on the local-global reward gap, steering each client toward prompts where it lags behind and away from those it has already mastered relative to the global average; and 
(iii)~\textbf{adaptive sparse communication} allocates bandwidth based on $\sigma_k$ and compresses updates by exploiting both the natural sparsity of GRPO updates and module-level update efficiency. 
Our motivation is illustrated in Figure \ref{fig:motivation}.
Extensive experiments across several mathematical reasoning benchmarks demonstrate that Fed-GRPO outperforms all federated baselines by a clear margin and closely matches the centralized upper bound.

We summarize our contributions as: i) We identify a fundamental mismatch between standard FL and GRPO, and reveal a powerful insight: the reward statistics ($\mu_k, \sigma_k$) inherently computed by GRPO provide zero-cost, perfect signals for federated coordination. ii) We propose \textbf{Fed-GRPO}, seamlessly integrating \emph{signal-weighted aggregation}, \emph{global reward calibration}, and \emph{adaptive sparse communication} to resolve the aggregation, optimization, and communication bottlenecks in federated GRPO training. iii) We provide a convergence analysis showing that Fed-GRPO converges at rate $\mathcal{O}(1/\sqrt{T})$ to a neighborhood of a stationary point of the calibrated objective, with the neighborhood characterized by two interpretable terms---a \emph{drift floor} from global reward calibration and a \emph{compression floor} from adaptive sparse communication. iv) We empirically validate Fed-GRPO on five math reasoning benchmarks, where it outperforms all federated baselines and approaches centralized GRPO training, losslessly reduces communication by $32\times$, and supports up to $621\times$ compression under tight bandwidth budgets with a graceful accuracy--bandwidth trade-off.





\section{Preliminary}

\subsection{Group Relative Policy Optimization}

Group Relative Policy Optimization (GRPO)~\cite{shao2024deepseekmath} is a critic-free LLM post-training method that computes advantages via group-relative reward comparisons, eliminating the memory-intensive value network required by PPO~\cite{schulman2017proximal}. Given a question $q$, GRPO samples a group of rollouts $\{o_1, o_2, \ldots, o_{G_{\text{rol}}}\}$ from the old policy $\pi_{\theta_\mathrm{old}}$ and optimizes the policy $\pi_\theta$ by maximizing the following objective:
{\footnotesize
\begin{equation}
\begin{aligned}
    \mathcal{J}_{\text{GRPO}}(\theta) = & \ \mathbb{E}_{q \sim \mathcal{D}, \{o_i\}_{i=1}^{G_{\text{rol}}} \sim \pi_{\theta_{old}}(\cdot|q)} \\
    &\left[ \frac{1}{G_{\text{rol}}}\sum_{i=1}^{G_{\text{rol}}} \frac{1}{|o_i|} \sum_{t=1}^{|o_i|} \left( \min\!\left( \rho_{i,t}(\theta) \hat{A}_{i},\; \text{clip}\left(\rho_{i,t}(\theta), 1{-}\epsilon, 1{+}\epsilon\right) \hat{A}_{i} \right) - \beta\, D_{\mathrm{KL}}\!\left(\pi_\theta \| \pi_{\mathrm{ref}}\right) \right) \right],
\end{aligned}
\label{eq:grpo_loss_J}
\end{equation}
}where $\rho_{i,t}(\theta) = \frac{\pi_\theta(o_{i,t}|q,o_{i,<t})}{\pi_{\theta_\mathrm{old}}(o_{i,t}|q,o_{i,<t})}$, $\epsilon$ is the clipping parameter, and $\beta$ controls the KL penalty strength. The advantage $\hat{A}_{i}$ is shared across all tokens in rollout $o_i$ and is computed by normalizing rewards across the group: 
\begin{equation}
    \hat{A}_{i} = \frac{r_i - \mu}{\sigma},
    \label{eq:grpo_advantage}
\end{equation}
where $r_i$ is the reward for rollout $o_i$, and $\mu$ and $\sigma$ are the mean and standard deviation of rewards within the group.

\subsection{Federated Learning}

Federated Learning (FL)~\cite{mcmahan2017communication,qu2022rethinking,yan2023label,pati2024privacy,zhang2024flhetbench} enables collaborative model training across $K$ clients without sharing raw data. In the standard FedAvg~\cite{mcmahan2017communication} framework, each communication round $t$ consists of three steps: (1) the server broadcasts the global model $\theta^{(t)}$ to all clients; (2) each client $k$ performs local training on its private dataset $\mathcal{D}_k$ to obtain updated parameters $\theta_k^{(t)}$ and sends them back to the server; (3) the server aggregates client updates via weighted averaging:
\begin{equation}
    \theta^{(t+1)} = \sum_{k=1}^{K} w_k \, \theta_k^{(t)},
    \label{eq:fedavg}
\end{equation}
where $w_k$ is the aggregation weight for client $k$, typically set to $\frac{|\mathcal{D}_k|}{\sum_{j=1}^{K} |\mathcal{D}_j|}$ or $\frac{1}{K}$ for uniform weighting. This process repeats until convergence or until reaching the maximum number of rounds.

\section{Analysis on Federated GRPO Training}
\label{sec:motivation}

To motivate our design, we conduct a preliminary experiment by training Qwen2.5-3B-Instruct~\cite{qwen2025qwen25} on the MATH dataset~\cite{hendrycks2021measuring} with a non-IID partition across $5$ clients based on problem difficulty (one client per difficulty level). We run standard FedAvg with GRPO as the local training algorithm for 12 rounds, with 5 local steps per round. Throughout training, we track the reward mean $\mu_k$ and standard deviation $\sigma_k$, averaged over all prompts of client $k$, as well as the overall model update sparsity. Figure~\ref{fig:analysis} summarizes the results, revealing three key observations that directly motivate our design.

\begin{figure}[h]
    \centering
        \begin{subfigure}{0.25\textwidth}
        \centering
        \includegraphics[width=\textwidth]{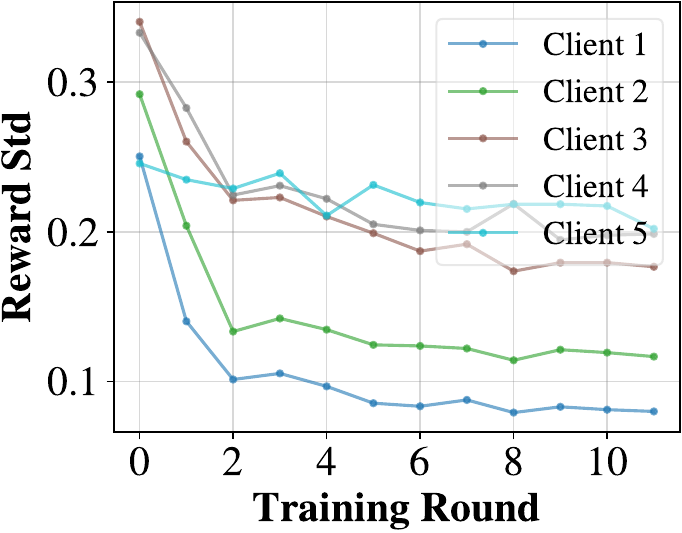}
        \caption{}
        \label{fig:mot_reward_std}
    \end{subfigure}\hfill
    \begin{subfigure}{0.25\textwidth}
        \centering
        \includegraphics[width=\textwidth]{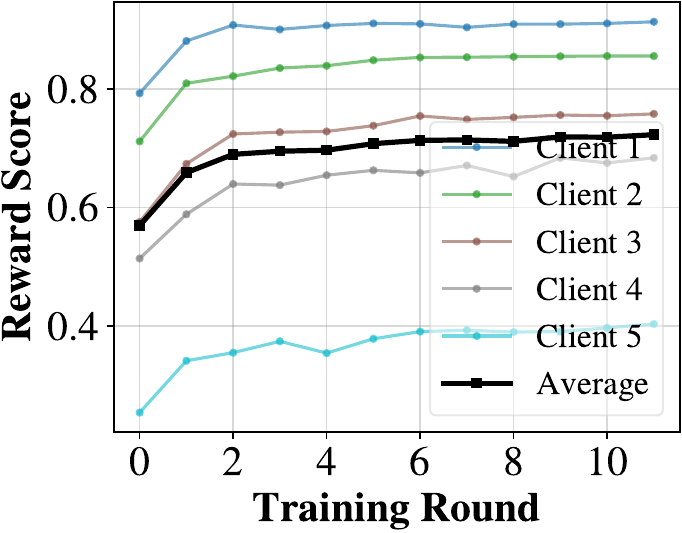}
        \caption{}
        \label{fig:mot_reward_score}
    \end{subfigure}\hfill
    \begin{subfigure}{0.25\textwidth}
        \centering
        \includegraphics[width=\textwidth]{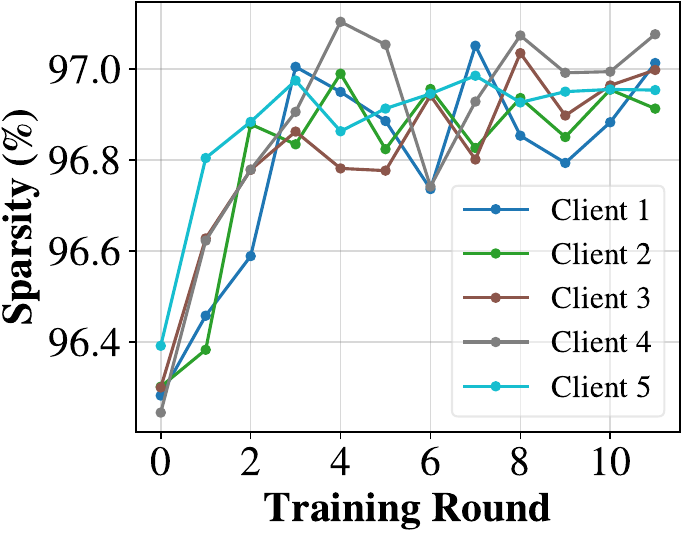}
        \caption{}
        \label{fig:mot_sparsity}
    \end{subfigure}\hfill
    \begin{subfigure}{0.25\textwidth}
        \centering
        \includegraphics[width=\textwidth]{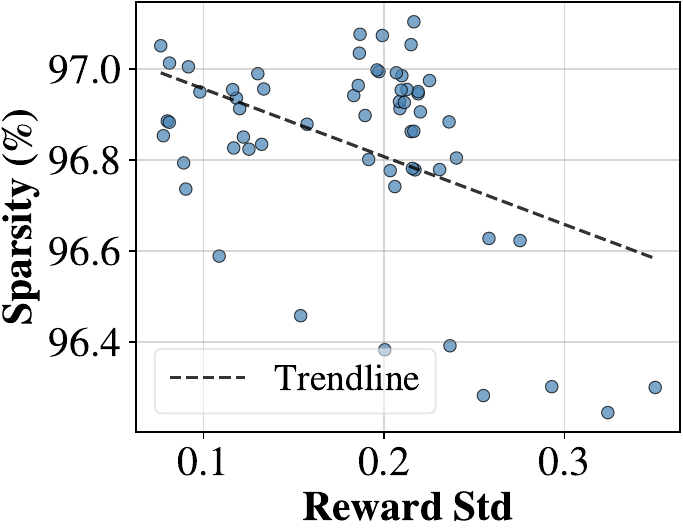}
        \caption{}
        \label{fig:mot_correlation}
    \end{subfigure}
    \caption{FedAvg-GRPO training dynamics on the MATH dataset with $5$ clients split by difficulty. (a)~Per-client reward standard deviation $\sigma_k$ varies by up to $4\times$ across clients within the same round, indicating large differences in learning signal strength. (b)~Per-client reward mean $\mu_k$ diverges significantly across clients, revealing a persistent local-global gap. (c)~Average model update sparsity across clients: over $96.2\%$ of parameters remain unchanged after local GRPO training. (d)~Correlation between $\sigma_k$ and model update sparsity: clients with stronger learning signals produce denser updates.}
    \label{fig:analysis}
\end{figure}

\textbf{Observation 1: Learning signals vary drastically across clients.} Figure~\ref{fig:mot_reward_std} shows that $\sigma_k$ varies by up to $4\times$ across clients within the same round. Since $\sigma_k$ reflects the overall strength of a client's learning signal in GRPO, such variation implies that actual gradient contributions also differ substantially. However, standard FedAvg ignores this and weights all clients equally by dataset size. \emph{This motivates signal-weighted aggregation}, which weights clients by $\sigma_k$ to prioritize those with stronger learning signals.

\textbf{Observation 2: Local reward means diverge from the global average.} Figure~\ref{fig:mot_reward_score} shows persistent gaps between per-client reward means $\mu_k$ and the global average throughout training. A client specializing in easy problems maintains a high $\mu_k$ but contributes little to global improvement, whereas a client on hard problems has a low $\mu_k$ yet holds the most room for improvement. Under standard GRPO, each client only sees its local reward statistics and cannot distinguish ``locally hard but globally easy'' prompts from ``locally easy but globally hard'' ones, resulting in misallocated training effort. \emph{This motivates global reward calibration}, which uses the gap between $\mu_k$ and the global reward mean to re-weight per-prompt objectives so that each client focuses on its relative weaknesses.

\textbf{Observation 3: GRPO updates are naturally sparse and correlate with $\sigma_k$.} Figure~\ref{fig:mot_sparsity} shows that over $96.2\%$ of parameters remain unchanged after local GRPO training, and this high sparsity is remarkably stable across all training rounds.
Moreover, Figure~\ref{fig:mot_correlation} reveals a strong negative correlation between $\sigma_k$ and model update sparsity---clients with stronger learning signals leave fewer parameters unchanged and thus produce denser updates. \emph{This motivates adaptive sparse communication}, which gives more bandwidth to high-$\sigma_k$ clients whose updates carry more information, while compressing low-$\sigma_k$ clients more aggressively, creating a natural alignment between update quality and communication cost.

Taken together, these observations show that the reward statistics $\mu_k$ and $\sigma_k$, computed by GRPO at \emph{no extra cost}, carry rich information about the aggregation, training, and communication needs of each client. In the next section, we describe how Fed-GRPO leverages these statistics to build a unified framework with three key mechanisms that jointly address the three challenges above.

\section{Methodology}
\label{sec:method}

Guided by the observations in Section~\ref{sec:motivation}, we design \textbf{Fed-GRPO}, a framework that leverages the reward statistics naturally produced during GRPO training to improve aggregation, local training, and communication for federated GRPO training. Fed-GRPO contains three key mechanisms: (i)~\textbf{signal-weighted aggregation} weights clients by $\sigma_k$ instead of dataset size, so that clients with stronger learning signals contribute more; (ii)~\textbf{global reward calibration} re-weights per-prompt objectives based on the local-global reward gap, steering each client toward prompts where it lags behind; and (iii)~\textbf{adaptive sparse communication} allocates bandwidth based on $\sigma_k$ and compresses updates by exploiting sparsity and efficiency-based module selection. We detail each mechanism below.

\subsection{Signal-Weighted Aggregation}
\label{sec:signal_weighted}

Standard FedAvg (Eq.~(\ref{eq:fedavg})) weights clients by dataset size (or uniformly), which ignores the \emph{quality} of each client's update during GRPO training. In GRPO training, dataset size is a poor proxy for update quality while the reward standard deviation $\sigma_k$ is a direct measure of the learning signal strength: a client with high $\sigma_k$ produces more informative policy gradients and thus should have a larger influence on the global update. To this end, we propose to weight clients by their reward standard deviation $\sigma_k$, which reflects the diversity of rewards and hence the strength of the learning signal. Specifically, the aggregation weight for client $k$ used in Eq.~(\ref{eq:fedavg}) is replaced by:
\begin{equation}
    w_k = \frac{\sigma_k}{\sum_{j=1}^{K} \sigma_j},
    \label{eq:signal_weighted}
\end{equation}
where $\sigma_k$ is the reward standard deviation from client $k$'s local GRPO training in the current round.

This weighting has two properties: (1) it is \textbf{zero-cost}, since $\sigma_k$ is already computed during GRPO training for advantage normalization; and (2) it is \textbf{adaptive}, meaning the aggregation weights adjust dynamically throughout the training process as each client's $\sigma_k$ naturally evolves (see Figure \ref{fig:swa_dynamics}).

\subsection{Global Reward Calibration}
\label{sec:reward_calibration}

In non-IID federated settings, clients have varying data distributions and varying levels of mastery. Under standard GRPO, each client optimizes its policy based solely on local reward outcomes, lacking any awareness of the global training progress. As a result, a client cannot objectively evaluate its own performance---it does not know whether a locally low reward implies an inherently hard problem or a severe lag behind the global model. This lack of global perspective leads to inefficient local training: clients waste valuable computation over-training on prompts they have already mastered relative to the global average, while under-training on their relative weaknesses.

To address this, we introduce \textbf{global reward calibration}, which uses the global reward statistics to re-weight per-prompt objectives during local training. After each aggregation round, the server computes global reward statistics from the clients' reported metrics and broadcasts them back to all clients:
{\scriptsize
\begin{equation}
    \mu_{\text{global}} = \frac{1}{K}\sum_{k=1}^{K} \mu_k, \qquad
    \sigma_{\text{global}} = \sqrt{\frac{1}{K-1}\sum_{k=1}^{K} (\mu_k - \mu_{\text{global}})^2}.
    \label{eq:global_stats}
\end{equation}
}

Each client then applies a per-prompt weight based on the gap between its local reward mean and the global mean, encouraging it to focus on prompts where it falls behind and reducing effort on prompts it has already mastered relative to others. Formally, for a prompt $q$ with local mean reward $\bar{r}_q$ (averaged over $G_{\text{rol}}$ rollouts), the calibration weight is:
\begin{equation}
    \omega_q = 1 + \lambda \cdot \tanh\!\left(\frac{\mu_{\text{global}} - \bar{r}_q}{\sigma_{\text{global}} + \varepsilon_{\sigma}}\right),
    \label{eq:prompt_weight}
\end{equation}
where $\lambda \in [0, 1]$ controls the calibration strength, and $\varepsilon_{\sigma}$ is a small constant to prevent division by zero. Then, the local GRPO training objective will be replaced by a weighted version that incorporates $\omega_q$:
\begin{equation}
    \mathcal{L}_k = \frac{\sum_q \omega_q \cdot \ell_q}{\sum_q \omega_q},
    \label{eq:calibrated_loss}
\end{equation}
where $k$ denotes the client index, and  $\ell_q := \frac{1}{G_{\text{rol}}}\sum_{i=1}^{G_{\text{rol}}}\frac{1}{|o_i|}\sum_{t}\big[\min(\rho_{i,t}\hat A_i, \mathrm{clip}(\rho_{i,t},1{-}\epsilon,1{+}\epsilon)\hat A_i)-\beta D_{\mathrm{KL}}\big]$ is the standard GRPO objective for prompt $q$.

The intuition behind this design is straightforward: when a client's local performance on a prompt is below the global average, the weight $\omega_q$ increases, amplifying the gradient contribution from that prompt and encouraging the client to focus on its weaknesses; conversely, when a client already exceeds the global average on a prompt, $\omega_q$ decreases, reducing unnecessary optimization on that prompt. The $\tanh$ function ensures that weights are bounded within $[1{-}\lambda, 1{+}\lambda]$, providing a smooth and stable calibration effect while preventing extreme values. In the first round, since global statistics are unavailable, all weights default to 1 and the method reduces to standard GRPO.

\subsection{Adaptive Sparse Communication}
\label{sec:compression}

The communication cost of transmitting full model parameters at each round is a major bottleneck in federated GRPO training, especially as model sizes grow. Fortunately, the natural sparsity of GRPO updates (over $96.2\%$ unchanged parameters, see Figure~\ref{fig:mot_sparsity}) and the correlation between $\sigma_k$ and update density (see Figure~\ref{fig:mot_correlation}) suggest an opportunity for \textbf{adaptive sparse communication} that allocates bandwidth based on the informativeness of each client's update. Specifically, clients with higher $\sigma_k$---whose updates carry more learning signal and receive larger aggregation weights---are given more communication bandwidth, while low-$\sigma_k$ clients (whose updates are mostly zero) are compressed more aggressively. This creates a natural alignment: low-signal clients are both compressed more and weighted less, minimizing the impact of compression-induced information loss.

Formally, given a target average per-client communication budget ratio $B \in (0, 1]$ (defined as the fraction of changed parameters to retain), we determine the adaptive budget ratio $B_k$ for client $k$ based on its normalized signal strength $\hat\sigma_k = \sigma_k / \max_j \sigma_j$ as follows:
\begin{equation}
    B_k = \text{clip}\left( B \cdot \big(1 + \alpha \cdot (\hat\sigma_k - \bar{\hat\sigma})\big), \; B_{\min}, \; 1 \right),
    \label{eq:adaptive_budget}
\end{equation}
where $\bar{\hat\sigma} = \frac{1}{K}\sum_j \hat\sigma_j$ is the mean normalized value across clients, $B_{\min}$ is a minimum budget to ensure that even low-signal clients can transmit some updates, and $\alpha$ controls the adaptation strength.

To meet the target retention ratio $B_k$, we apply a two-stage compression strategy that exploits the structural properties of GRPO updates. First, we apply a \emph{lossless} sparse encoding that extracts only the parameters that changed after local training, leveraging the fact that most parameters remain unchanged during GRPO training (see Figure~\ref{fig:mot_sparsity}). This alone achieves significant compression (approximately $32\times$) without any loss in accuracy, keeping $100\%$ of the updated parameters ($B_k=1.0$). When $B_k < 1.0$, we further apply lossy compression by selecting a subset of modules to transmit based on their \emph{update efficiency}, ensuring that we preserve the most informative updates while dropping parameters until reaching the target fraction $B_k$. The update efficiency of a module is defined as the update energy per transmitted parameter:
\begin{equation}
    E_m = \frac{\|\Delta_m\|_F^2}{\|\Delta_m\|_0},
    \label{eq:efficiency}
\end{equation}
where $\Delta_m$ is the update for module $m$ and $\|\Delta_m\|_0$ is the number of changed parameters in module $m$. We then sort modules by efficiency $E_m$ in increasing order and greedily remove modules until the final transmitted parameter count roughly equals the fraction $B_k$ of the changed parameters. 
This efficiency-based selection ensures that we drop only the least informative updates, maximizing the retained learning signal under tight communication constraints. Importantly, since $B_k$ is adaptive to $\sigma_k$, clients with stronger learning signals are more likely to retain more of their updates, while clients with weaker signals (which contribute less to the global update) are compressed more aggressively, creating a natural synergy between update quality and communication cost.

\subsection{Fed-GRPO Algorithm}

Fed-GRPO integrates the above three mechanisms into a unified training procedure. Specifically, each round begins with the server broadcasting the current global model and global reward statistics to all clients. Each client then performs local GRPO training with global reward calibration, computes its reward statistics, and applies adaptive sparse communication to transmit its compressed update back to the server. The server aggregates the received updates using signal-weighted aggregation and updates the global reward statistics for the next round. This iterative process continues until convergence or until reaching the maximum number of rounds. Algorithm~\ref{alg:fedgrpo} in Appendix~\ref{app_sec:algo} provides a detailed summary.


\section{Convergence Analysis}
\label{sec:convergence}

We make the following assumptions, some of which are commonly used in the FL literature~\cite{Li2020on,guo2025selective,karimireddy2019error,guo2025new}, while Assumption~\ref{assumption:drift} adapts a tool from non-stationary stochastic optimization~\cite{besbes2015non,jadbabaie2015online}.

\begin{assumption} (L-smoothness)
\label{assumption:l-smooth}
Each local objective $\mathcal{L}_k$ is $L$-smooth: for any parameters $\theta, \theta'$,
\begin{equation*}
    \mathcal{L}_k(\theta) \le \mathcal{L}_k(\theta') + \langle \theta - \theta',\, \nabla \mathcal{L}_k(\theta')\rangle_F + \frac{L}{2}\|\theta - \theta'\|^2_F,
\end{equation*}
where $\langle\cdot, \cdot\rangle_F$ and $\|\cdot\|_F$ denote the Frobenius inner product and norm, respectively.
\end{assumption}

\begin{assumption} (Bounded stochastic gradient)
\label{assumption:bounded_gradient}
Let $\xi_{k,s}^{(t)}$ be sampled from the $k$-th client's local data uniformly at random at the $s$-th training step of round $t$. The expected squared norm of stochastic gradients is uniformly bounded: $\mathbb{E}\|\nabla \mathcal{L}_k^{(t)}(\theta_{k,s}^{(t)};\xi_{k,s}^{(t)})\|_F^2 \le G^2$, for all $k=1, \cdots, K$, $t=0, \cdots, T-1$, and $s=1, \cdots, S$.
\end{assumption}

\begin{assumption} (Bounded compression error)
\label{assumption:bounded_compression}
The adaptive sparse communication operator $\mathcal{C}_k(\cdot)$ in Section~\ref{sec:compression} is contractive path-wise: $\|\mathcal{C}_k(\Delta) - \Delta\|_F^2 \le (1-B_k)\|\Delta\|_F^2$ for any $\Delta$, where $B_k$ is the budget ratio assigned to client $k$.
\end{assumption}

\begin{assumption} (Bounded average calibration drift)
\label{assumption:drift}
The average round-to-round drift of the global objective along the iterate path is bounded in expectation:
\begin{equation*}
    \frac{1}{T}\sum_{t=0}^{T-1} \mathbb{E}[\Delta^{(t)}_{\mathcal{L}}] \;\le\; R,
    \qquad
    \Delta^{(t)}_{\mathcal{L}} \;:=\; \bigl|\mathcal{L}^{(t+1)}(\theta^{(t+1)}) - \mathcal{L}^{(t)}(\theta^{(t+1)})\bigr|.
\end{equation*}
Here $\mathcal{L}^{(T)}$ denotes any uniformly bounded extension of the round-statistics (e.g., reusing the round-$(T{-}1)$ weights and calibration); the specific choice does not affect the bound.
\end{assumption}

Then, we present the convergence rate for Fed-GRPO, with the proof provided in Appendix~\ref{app:proof_convergence}.

\begin{theorem}
\label{thm:convergence}
Let Assumptions \ref{assumption:l-smooth}--\ref{assumption:drift} hold and $L$, $G$, $R$ be defined therein. For a learning rate $\eta$, we have:
\begin{equation}
        \frac{1}{T}\sum_{t=0}^{T-1} \mathbb{E}\bigl\|\nabla \mathcal{L}^{(t)}(\theta^{(t)})\bigr\|_F^2 \;\le\; {2\sqrt{\frac{D\,M}{T}}} \;+\; 2\sqrt{R\,M} \;+\; {4\sqrt{K(1-\bar{B})}\,G^2},
        \label{eq:thm1_main}
\end{equation}
\emph{where $\mathcal{L}^{(t)}(\theta) = \sum_{k=1}^K w_k^{(t)} \mathcal{L}_k^{(t)}(\theta)$ is the global objective at round $t$, $D := \mathcal{L}^{(0)}(\theta^{(0)}) - \inf_{0\le t\le T,\,\theta}\mathcal{L}^{(t)}(\theta)$, $\bar{B} = \tfrac{1}{KT}\sum_{t=0}^{T-1}\sum_{k=1}^K B_k^{(t)}$ is the per-client budget ratio averaged over both clients and rounds (which equals the user-specified target $B$ in Eq.~(\ref{eq:adaptive_budget}) when clipping is inactive), and $M$ is a constant depending on $L, G, S, K$.}
\end{theorem}

According to Theorem \ref{thm:convergence}, Fed-GRPO converges to a neighborhood of a stationary point of the time-varying calibrated objective at rate $\mathcal{O}(1/\sqrt{T})$, where the neighborhood size is determined by two irreducible terms: a \emph{drift floor} of size $\mathcal{O}(\sqrt{R M})$ from the round-to-round change of the calibrated objective, and a \emph{compression floor} of size $\mathcal{O}(\sqrt{K(1-\bar{B})}\, G^2)$ from adaptive sparse communication. 
When the time-varying global objective is stationary across rounds ($R=0$) and the encoding is lossless ($\bar{B} = 1$), both floors vanish and Fed-GRPO recovers the standard $\mathcal{O}(1/\sqrt{T})$ rate of FedAvg in the non-convex setting~\cite{Li2020on}. 
\section{Experiments}
\label{sec:experiments}

\subsection{Experimental Setup}
\label{sec:exp_details}

\pengxin{For mathematical reasoning tasks tasks,}
we train on the MATH dataset~\cite{hendrycks2021measuring}, split across $K{=}5$ clients by difficulty to form a naturally non-IID partition, and evaluate on five math reasoning benchmarks: GSM8K~\cite{cobbe2021training}, MATH-500~\cite{hendrycks2021measuring,lightman2023let}, AIME 2024, AIME 2025, and AMC 2023. The base models are Qwen2.5-3B-Instruct~\cite{qwen2025qwen25} \pengxin{and Qwen3-4B-Instruct~\cite{yang2025qwen3}}, optimized with AdamW~\cite{loshchilov2018decoupled} (lr $3{\times}10^{-6}$) under GRPO with a binary correctness reward, prompt batch size $64$, and $8$ rollouts per prompt. 
\pengxin{For coding tasks, we use Qwen2.5-3B-Instruct as the base model, train on the MBPP dataset~\cite{austin2021program} ("train" + "validation") and evaluate on the held-out MBPP "test" split and the HumanEval dataset~\cite{chen2021evaluating}. To construct a non-IID federation setting, we partition the MBPP dataset by the reference-solution length, bucket it into 5 quantile bins, and partition the data across clients with a Dirichlet distribution ($\alpha=0.5$) over these buckets.}
All federated methods run for $T{=}12$ communication rounds with $S{=}5$ local steps; centralized and local-only baselines are trained for an equivalent number of gradient updates. We compare Fed-GRPO against the base model, centralized GRPO (upper bound), local-only GRPO, sample- and uniform-weighted FedAvg~\cite{mcmahan2017communication} (FedAvg-SW / FedAvg-UW), two sparse-communication FL methods, SparsyFed~\cite{guastella2025sparsyfed} and SparseLoCo~\cite{sarfi2025communication}, \pengxin{and one Federated GRPO method, FGRPO~\cite{chen2026fgrpo}}. Following previous GRPO literature~\cite{yu2025dapo, guo2025deepseek}, we evaluate the best checkpoint of each method and report avg@$32$ (32 samples per prompt at temperature $1.0$, top-$p{=}0.7$). More details are deferred to Appendix~\ref{app_sec:full_exp_details}.

\subsection{Main Results}

\begin{table}[h]
\caption{Main results on five math reasoning benchmarks \pengxin{using Qwen2.5-3B-Instruct as the base model}. FedAvg-SW and FedAvg-UW denote sample-weighted and uniform-weighted FedAvg, respectively. Budget $B$ is the target fraction of changed parameters retained for transmission; Comm.\ (MB) is the average per-round per-client uplink payload (BF16 values plus 4-byte indices for sparse encodings); Comp.\ ($\times$) reports the compression ratio relative to dense full-model transmission as \emph{parameter ratio / byte ratio}. Best result for federated methods in \textbf{bold}. $^{\dagger}$For SparsyFed and SparseLoCo, Budget denotes the kept fraction of \emph{all} parameters.}
\label{tab:main_results}
\vspace{-0.3cm}
\begin{center}
\resizebox{\linewidth}{!}{
\begin{tabular}{lccc|ccccc|c}
\toprule
\textbf{Method} & \textbf{Budget} & \textbf{Comm.\,(MB)} & \textbf{Comp.\,($\times$)} & \textbf{GSM8K} & \textbf{MATH-500} & \textbf{AIME24} & \textbf{AIME25} & \textbf{AMC} & \textbf{Avg.} \\
\midrule
Qwen2.5-3B-Instruct & - & - & - & 59.47 & 46.54 & 2.50 & 1.46 & 27.11 & 27.42 \\
\midrule
Centralized & - & - & - & 84.75 & 63.42 & 7.50 & 3.44 & 41.95 & 40.21 \\
\midrule
Local-Only & - & - & - & 82.45 & 58.16 & 5.88 & 2.94 & 38.28 & 37.54 \\
FedAvg-SW & - & 6{,}794.2 & $1\times\,/\,1\times$ & 79.55 & 56.84 & 6.56 & 3.02 & 39.53 & 37.10 \\
FedAvg-UW & - & 6{,}794.2 & $1\times\,/\,1\times$ & 80.00 & 55.73 & 5.42 & \textbf{3.12} & 37.66 & 36.39 \\
SparsyFed & $5\%^{\dagger}$  & 1{,}019.1 & $20\times\,/\,7\times$  & 79.10 & 56.61 & 6.56 & 2.92 & 39.53 & 36.94 \\
SparseLoCo & $3\%^{\dagger}$ & 451.2     & $32\times\,/\,15\times$ & 78.69 & 56.05 & 6.04 & 2.81 & 39.30 & 36.58 \\
FGRPO & - & 6{,}794.2 & $1\times\,/\,1\times$ & 83.11 & 59.79 & 6.83 & 3.02 & 39.73 & 38.49 \\
\cmidrule{1-10}
\multirow{5}{*}{Fed-GRPO (Ours)} & $100\%$ & 645.5 & $32\times\,/\,11\times$ & \textbf{83.21} & 60.82 & 7.08 & \textbf{3.12} & \textbf{41.09} & \textbf{39.06} \\
& $50\%$  & 321.4 & $63\times\,/\,21\times$   & 83.11 & \textbf{60.84} & \textbf{7.19} & 2.71 & 41.25 & 39.02 \\
& $20\%$  & 126.9 & $161\times\,/\,54\times$  & 81.16 & 56.14 & 5.31 & 2.81 & 38.67 & 36.82 \\
& $10\%$  & 62.2  & $327\times\,/\,109\times$ & 79.69 & 55.21 & 5.42 & 3.12 & 36.25 & 35.94 \\
& $5\%$   & 32.8  & $621\times\,/\,207\times$ & 77.29 & 54.83 & 5.62 & 3.23 & 34.45 & 35.08 \\
\bottomrule
\end{tabular}
}
\end{center}
\end{table}

Table~\ref{tab:main_results} reports the main results achieved by each method across the five benchmarks. From this table, we can draw several key conclusions. First, Fed-GRPO achieves the best performance among all federated methods, and approaches the centralized upper bound, demonstrating that it effectively mitigates the challenges of non-IID data and aggregation mismatch in federated GRPO training. Second, the fact that FedAvg-SW and FedAvg-UW perform worse than Local-Only highlights the pitfalls of naive aggregation strategies that ignore update quality, underscoring the importance of our signal-weighted aggregation mechanism. Third, Fed-GRPO offers a flexible communication--accuracy trade-off across a wide range of bandwidth budgets. At $B{=}100\%$, our adaptive sparse communication is mathematically lossless and already shrinks the per-round uplink from $6{,}794.2$\,MB to $645.5$\,MB---a $32\times$ parameter ($11\times$ byte) compression with \emph{zero} accuracy loss. Halving the budget to $B{=}50\%$ doubles the compression to $63\times$ ($21\times$ in bytes) while preserving essentially the same average accuracy ($39.02$ vs.\ $39.06$), and even at $B{=}20\%$ ($161\times$, $54\times$ in bytes) Fed-GRPO still matches FedAvg's accuracy at a small fraction of the bandwidth. Pushing further to $B{=}10\%$ and $B{=}5\%$ delivers extreme $327\times$ and $621\times$ parameter compression ($109\times$ and $207\times$ in bytes; $32.8$\,MB per round at $B{=}5\%$), at the cost of a moderate accuracy drop of $\sim$3--4 points relative to lossless transmission. We emphasize that this $621\times$ regime is intended for severely bandwidth-constrained deployments rather than as the default operating point: as the lower rows of Table~\ref{tab:main_results} show, the degradation is gradual rather than abrupt, allowing practitioners to dial $B$ along a smooth Pareto frontier and select an operating point---e.g., $B{=}50\%$ as a near-lossless sweet spot, or $B{=}5\%$ when bandwidth dominates---according to their bandwidth--accuracy preferences.

{\color{black}
\subsection{More Results on Other Base Models and Tasks}

Tables~\ref{tab:main_results_qwen3_4b} and \ref{tab:main_results_coding} report additional results on math reasoning benchmarks using Qwen3-4B-Instruct as the base model and on coding tasks using Qwen2.5-3B-Instruct as the base model, respectively. The results show that Fed-GRPO consistently outperforms all other federated methods across different base models and tasks, demonstrating its generality and effectiveness in various settings. More results on different data heterogeneity and analysis on learning dynamics are provided in Appendix~\ref{sec:more_exp_results}.

\begin{table}[h]

\begin{minipage}[t]{0.65\linewidth}
\caption{Results on five math reasoning benchmarks using Qwen3-4B-Instruct as the base model.}
\label{tab:main_results_qwen3_4b}
\vspace{-0.3cm}
\begin{center}
\resizebox{\linewidth}{!}{
\begin{tabular}{lccccc|c}
\toprule
\textbf{Method} & \textbf{GSM8K} & \textbf{MATH-500} & \textbf{AIME24} & \textbf{AIME25} & \textbf{AMC} & \textbf{Avg.} \\
\midrule
Qwen3-4B-Instruct & 77.56 & 84.55 & 32.7 & 29.79 & 77.5 & 60.42  \\
\midrule
Centralized & 93.76 & 86.90 & 36.67 & 32.50 & 81.95 & 66.36 \\
\midrule
Local-Only & 93.45 & 86.18 & 33.75 & 30.04 & 78.55 & 64.39 \\
FedAvg-SW & 93.66 & 86.76 & 35.62 & 29.79 & 82.50 & 65.66 \\
FedAvg-UW & 93.73 & 86.73 & 34.90 & 30.52 & 81.02 & 65.38 \\
SparsyFed & 93.48 & 86.53 & 35.02 & 31.22 & 81.13 & 65.47 \\
SparseLoCo & 93.28 & 86.32 & 34.93 & 31.03 & 81.07 & 65.33 \\
FGRPO & 93.77 & 86.75 & 35.23 & 31.78 & 81.62 & 65.83 \\
\cmidrule{1-7}
Fed-GRPO (Ours) & \textbf{93.93} & \textbf{86.85} & \textbf{35.94} & \textbf{32.29} & \textbf{82.73} & \textbf{66.35} \\
\bottomrule
\end{tabular}
}
\end{center}
\end{minipage}
\hfill
\begin{minipage}[t]{0.32\linewidth}
\caption{Results on coding tasks using Qwen2.5-3B-Instruct as the base model.}
\label{tab:main_results_coding}
\vspace{-0.3cm}
\begin{center}
\resizebox{\linewidth}{!}{
\begin{tabular}{lcc}
\toprule
\textbf{Method} & \textbf{GSM8K} & \textbf{MATH-500} \\
\midrule
Qwen2.5-3B-Instruct & 57.40 & 60.41 \\
\midrule
Centralized & 62.80 & 62.85 \\
\midrule
Local-Only & 59.60 & 60.80 \\
FedAvg-SW & 60.70 & 61.63 \\
FedAvg-UW & 60.82 & 61.82 \\
SparsyFed & 60.66 & 61.23 \\
SparseLoCo & 60.03 & 60.98 \\
FGRPO & 61.42 & 61.94 \\
\cmidrule{1-3}
Fed-GRPO (Ours) & \textbf{62.23} & \textbf{62.44} \\
\bottomrule
\end{tabular}
}
\end{center}
\end{minipage}

\end{table}

}

\subsection{Ablation Studies}
\label{sec:ablation}

To isolate the contribution of each mechanism in Fed-GRPO, we conduct ablation studies \pengxin{on the math tasks using Qwen2.5-3B-Instruct as the base model} by starting from FedAvg-SW (i.e., Fed-GRPO with all three mechanisms disabled) and progressively adding each component. As shown in Table~\ref{tab:ablation}, each mechanism contributes complementarily. Starting from FedAvg-SW (37.10\%), \emph{signal-weighted aggregation} (SWA) lifts the average to 38.01\% ($+0.91$) with consistent gains across all five benchmarks, confirming that $\sigma_k$ better reflects update quality than dataset size. Adding \emph{global reward calibration} (GRC) further improves the average to 39.07\% ($+1.06$), with the largest gains on harder benchmarks (e.g., MATH-500: $58.73 \to 61.26$), matching GRC's design intent of steering clients toward globally lagging prompts. Finally, \emph{adaptive sparse communication} (ASC) at $B{=}1.0$ preserves accuracy (39.06\% vs.\ 39.07\%) while shrinking the per-client uplink from $6{,}794.2$\,MB to $645.5$\,MB---a $11\times$ reduction at zero accuracy cost. 
Overall, SWA and GRC drive the accuracy gains, while ASC contributes orthogonal communication savings.

\begin{table}[h]
\caption{Ablation studies on Fed-GRPO components. $+$SWA: add signal-weighted aggregation. $+$GRC: add global reward calibration. $+$ASC: add adaptive sparse communication (at $B{=}1.0$, lossless). Comm. (MB) refers to the average per-round per-client uplink payload.}
\label{tab:ablation}
\vspace{-0.3cm}
\begin{center}
\resizebox{\linewidth}{!}{
\begin{tabular}{lcccccc|c}
\toprule
\textbf{Configuration} & \textbf{GSM8K} & \textbf{MATH-500} & \textbf{AIME24} & \textbf{AIME25} & \textbf{AMC} & \textbf{Avg.} & \textbf{Comm. (MB)} \\
\midrule
FedAvg-SW (no component) & 79.55 & 56.84 & 6.56 & 3.02 & 39.53 & 37.10 & 6{,}794.2 \\
$+$ SWA (\S\ref{sec:signal_weighted}) & 81.24 & 58.73 & 7.02 & 3.05 & 40.03 & 38.01 & 6{,}794.2 \\
$+$ SWA $+$ GRC (\S\ref{sec:reward_calibration}) & 82.75 & 61.26 & 7.26 & 3.13 & 40.94 & 39.07 & 6{,}794.2 \\
$+$ SWA $+$ GRC $+$ ASC (Fed-GRPO) & 83.21 & 60.82 & 7.08 & 3.12 & 41.09 & 39.06 & 645.5 \\
\bottomrule
\end{tabular}
}
\end{center}
\end{table}

\section{Related Work}
\label{sec:related_work}

Several works also explore federated GRPO training but most from angles orthogonal to ours.
FedGRPO~\cite{zhu2026fedgrpo} guides privacy-preserving server updates using GRPO's "Group Relative" rewards among experts based on clients' domain knowledge. FedMOA~\cite{wang2026fedmoa} tackles multi-objective RL via hypergradient-based online weighting and task-aware server aggregation. \pengxin{Concurrently, FGRPO~\cite{chen2026fgrpo} introduces an adaptive aggregation mechanism that dynamically prioritizes effective learning trajectories based on each client's performance gain relative to its historical baseline.}
Differently, Fed-GRPO targets the standard federated GRPO reasoning with a shared accuracy-style reward and non-IID prompt distribution. It is the first to identify that the reward statistics $(\mu_k,\sigma_k)$ produced ``for free'' during GRPO advantage normalization carry rich information about each client's aggregation, training, and communication needs---turning them into a unified, zero-cost signal that drives all three mechanisms.

\section{Conclusion}

We present Fed-GRPO, a novel FL framework for training large language models with GRPO. Fed-GRPO introduces three key mechanisms that leverage GRPO's unique reward statistics to address the challenges of naïve Fedavg GRPO training: \emph{signal-weighted aggregation} that dynamically weights client updates based on their reward standard deviation, \emph{global reward calibration} that re-weights local objectives based on local-global reward gaps, and \emph{adaptive sparse communication} that allocates bandwidth based on update quality. Experiments on several mathematical reasoning tasks demonstrate that Fed-GRPO significantly outperforms federated baselines and approaches centralized performance while achieving up to $32\times$ lossless communication compression. This work opens up new possibilities for collaborative reasoning training across distributed data sources without sharing raw data.

\bibliography{ref}
\bibliographystyle{iclr2027_conference}

\clearpage
\appendix

\section{Fed-GRPO Algorithm}
\label{app_sec:algo}

Algorithm~\ref{alg:fedgrpo} provides a detailed summary of the Fed-GRPO method.

\begin{algorithm}[h]
\caption{Fed-GRPO}
\label{alg:fedgrpo}
\begin{algorithmic}[1]
\renewcommand{\algorithmicrequire}{\textbf{Input:}}
\renewcommand{\algorithmicensure}{\textbf{Output:}}
\REQUIRE Initial model $\theta^{(0)}$, local datasets $\{\mathcal{D}_k\}_{k=1}^K$, communication rounds $T$, local training steps $S$, target communication budget ratio $B$
\ENSURE Trained global model $\theta^{(T)}$

\STATE \textbf{Server initializes:} $\mu_{\text{global}} \gets \text{None}$, $\sigma_{\text{global}} \gets \text{None}$

\FOR{$t = 0, 1, \ldots, T-1$}
    \STATE \textbf{Server} broadcasts global model $\theta^{(t)}$ and statistics $(\mu_{\text{global}}, \sigma_{\text{global}})$ to all clients
    \FOR{each client $k \in \{1, \ldots, K\}$ \textbf{in parallel}}
        \STATE $\theta_k \gets \theta^{(t)}$ \hfill $\triangleright$ \textit{Initialize local model}
        \FOR{$s = 1, \ldots, S$}
            \STATE Sample a batch of prompts from $\mathcal{D}_k$; generate $G_{\text{rol}}$ rollouts per prompt
            \STATE Compute rewards and advantages $\{\hat{A}_i\}$ for the group (Eq.~(\ref{eq:grpo_advantage}))
            \IF{$\mu_{\text{global}} \neq \text{None}$}
                \STATE Compute calibration weights $\{\omega_q\}$ based on local--global reward gap (Eq.~(\ref{eq:prompt_weight}))
            \ELSE
                \STATE $\omega_q \gets 1$ for all prompts
            \ENDIF
            \STATE Update $\theta_k$ via the calibrated GRPO loss $\mathcal{L}_k$ (Eq.~(\ref{eq:calibrated_loss}))
        \ENDFOR
        \STATE Compute local reward mean $\mu_k$ and standard deviation $\sigma_k$ over the round
        \STATE Compute local model update $\Delta_k \gets \theta_k - \theta^{(t)}$
        \STATE Determine adaptive budget ratio $B_k$ (Eq.~(\ref{eq:adaptive_budget})) \hfill $\triangleright$ \textit{Allocate bandwidth}
        \STATE Let $\tilde{\Delta}_k$ be the lossless sparse encoding of $\Delta_k$
        \IF{$B_k < 1.0$}
            \STATE Further compress $\tilde{\Delta}_k$: retain top modules by efficiency $E_m$ (Eq.~(\ref{eq:efficiency})) until we only retain a fraction $B_k$ of the changed parameters
        \ENDIF
        \STATE \textbf{Client} sends compressed update $\tilde{\Delta}_k$ and reward statistics $(\mu_k, \sigma_k)$ to server
    \ENDFOR
    \STATE \textbf{Server} computes signal-weighted aggregation weights $\{w_k\}$ (Eq.~(\ref{eq:signal_weighted}))
    \STATE \textbf{Server} decompresses each $\tilde{\Delta}_k$ back to a dense delta $\Delta_k$
    \STATE \textbf{Server} applies aggregated updates: $\theta^{(t+1)} = \theta^{(t)} + \sum_{k=1}^{K} w_k \, \Delta_k$
    \STATE \textbf{Server} updates global statistics $\mu_{\text{global}}$ and $\sigma_{\text{global}}$ (Eq.~(\ref{eq:global_stats}))
\ENDFOR

\RETURN $\theta^{(T)}$
\end{algorithmic}
\end{algorithm}

\section{Proof of Theorem~\ref{thm:convergence}}
\label{app:proof_convergence}

\begin{proof}
Let $\theta^{(t)}$ be the global model at the beginning of round $t$, and let $\theta_k^{(t,s)}$ be the local model of client $k$ after $s$ local steps within round $t$, with $\theta_k^{(t,0)} = \theta^{(t)}$. Within round $t$, the one-round update of the global model can be described as the following two sub-steps:
\begin{equation*}
    \theta^{(t)} \xrightarrow[\text{local update of each client}]{}  \big\{\theta_k^{(t,S)}\big\}_{k=1}^K \xrightarrow[\text{server aggregation}]{}  \theta^{(t+1)}.
\end{equation*}
For convenience, we denote the intermediate quantities as follows:
\begin{align*}
    & \theta_k^{(t,s+1)} = \theta_k^{(t,s)} - \eta g_k^{(t,s)}, && s = 0,1,\ldots,S-1, \\
    & \Delta_k^{(t)} = \theta_k^{(t,S)} - \theta^{(t)} = -\eta \sum_{s=0}^{S-1} g_k^{(t,s)}, \\
    & U^{(t)} = \theta^{(t)} + \sum_{k=1}^K w_k^{(t)} \Delta_k^{(t)}, \\
    & \theta^{(t+1)} = U^{(t)} + \sum_{k=1}^K w_k^{(t)} \big(\tilde{\Delta}_k^{(t)} - \Delta_k^{(t)}\big),
\end{align*}
where $g_k^{(t,s)}$ is the stochastic gradient of the round-$t$ calibrated local objective $\mathcal{L}_k^{(t)}$ on a mini-batch $\xi_{k,s}^{(t)}$, with $\mathbb{E}\big[g_k^{(t,s)}\big] = \nabla \mathcal{L}_k^{(t)}\big(\theta_k^{(t,s)}\big)$ and $\mathbb{E}\big\|g_k^{(t,s)}\big\|_F^2 \le G^2$ by Assumption~\ref{assumption:bounded_gradient}; $\tilde{\Delta}_k^{(t)} = \mathcal{C}_k(\Delta_k^{(t)})$ is the compressed update produced by the adaptive sparse communication operator in Section~\ref{sec:compression}; and $w_k^{(t)} = \sigma_k^{(t)}/\sum_j \sigma_j^{(t)}$ is the signal-weighted aggregation weight in Eq.~(\ref{eq:signal_weighted}). Here, $U^{(t)}$ represents the immediate result obtained after each client transmits its uncompressed local update, and $\theta^{(t+1)}$ represents the parameter obtained after the actual compressed-communication step.

Since each $\mathcal{L}_k^{(t)}$ is $L$-smooth (Assumption~\ref{assumption:l-smooth}) and $\sum_k w_k^{(t)} = 1$, the global objective $\mathcal{L}^{(t)} = \sum_{k=1}^K w_k^{(t)}\mathcal{L}_k^{(t)}(\theta)$ is also $L$-smooth.
By the smoothness of $\mathcal{L}^{(t)}$,
\begin{equation}
    \mathcal{L}^{(t)}(U^{(t)}) \le \mathcal{L}^{(t)}(\theta^{(t)}) + \big\langle U^{(t)} - \theta^{(t)}, \nabla \mathcal{L}^{(t)}(\theta^{(t)})\big\rangle_F + \frac{L}{2}\big\|U^{(t)} - \theta^{(t)}\big\|_F^2.
    \label{eq:smooth_U}
\end{equation}
For the squared-norm term in Eq.~\eqref{eq:smooth_U}, since $U^{(t)} - \theta^{(t)} = -\eta \sum_k w_k^{(t)} \sum_s g_k^{(t,s)}$, by 
Cauchy-Schwarz inequality,
\begin{align*}
    \big\|U^{(t)} - \theta^{(t)}\big\|_F^2  
    = & \eta^2 \big\|\sum_k w_k^{(t)} \sum_s g_k^{(t,s)}\big\|_F^2
    \le \eta^2 \big(\sum_k (w_k^{(t)})^2 \big) \Big(\sum_k \big\|\sum_{s=0}^{S-1} g_k^{(t, s)}\big\|_F^2\Big) \\
    \le & \eta^2 \big(\sum_k w_k^{(t)} \big) \Big(\sum_k S \sum_{s=0}^{S-1} \|g_k^{(t, s)}\|_F^2\Big)
    \le \eta^2 S \sum_{k=1}^K \sum_{s=0}^{S-1} \big\|g_k^{(t,s)}\big\|_F^2,
\end{align*}
where we use the fact that $w_k^{(t)} \le 1$ for the last second step, and $\sum_k w_k^{(t)} = 1$ for the last step.
Taking expectation,
\begin{equation}
    \mathbb{E}\big\|U^{(t)} - \theta^{(t)}\big\|_F^2 \le \eta^2 S \sum_{k=1}^K \sum_{s=0}^{S-1} \mathbb{E}\big\|g_k^{(t,s)}\big\|_F^2 \le K \eta^2 S^2 G^2.
    \label{eq:U_W_norm}
\end{equation}
For the inner-product term in Eq.~\eqref{eq:smooth_U}, 
\begin{align*}
\big\langle U^{(t)} - \theta^{(t)}, \nabla \mathcal{L}^{(t)}(\theta^{(t)})\big\rangle_F
= -\eta \big\langle \sum_k w_k^{(t)} \sum_s g_k^{(t,s)}, \nabla \mathcal{L}^{(t)}(\theta^{(t)})\big\rangle_F,
\end{align*}
taking expectation over the stochastic gradients in current round (Denote with $\mathbb{E}_t$), and using $\mathbb{E}[X Y]=\mathbb{E}[X] \mathbb{E}[Y]+\operatorname{Cov}(X, Y)$, we get 
\begin{align*}
& \mathbb{E}_t\big[\big\langle \sum_k w_k^{(t)} \sum_s g_k^{(t,s)}, \nabla \mathcal{L}^{(t)}(\theta^{(t)})\big\rangle_F\big]\\
= & \big\langle \sum_k \mathbb{E}_t[w_k^{(t)}] \sum_s \mathbb{E}_t[g_k^{(t,s)}], \nabla \mathcal{L}^{(t)}(\theta^{(t)})\big\rangle_F
+ \sum_{k} \sum_{s} \big\langle \operatorname{Cov}_t\big(w_k^{(t)}, g_k^{(t,s)}\big), \nabla \mathcal{L}^{(t)}(\theta^{(t)})\big \rangle\\
= & \sum_k \mathbb{E}_t[w_k^{(t)}] \sum_s  \big\langle \nabla \mathcal{L}_k^{(t)}(\theta_k^{(t,s)}), \nabla \mathcal{L}^{(t)}(\theta^{(t)})\big\rangle_F
+ \sum_{k} \sum_{s} \big\langle \operatorname{Cov}_t\big(w_k^{(t)}, g_k^{(t,s)}\big), \nabla \mathcal{L}^{(t)}(\theta^{(t)})\big \rangle
\end{align*}
If we assume $\big\|\operatorname{Cov}_t\big(w_k^{(t)}, g_k^{(t,s)}\big)\big\|_F \leq \delta_t$ for any $k,s$, then 
\begin{align*}
& \sum_{k} \sum_{s} \big\langle \operatorname{Cov}_t\big(w_k^{(t)}, g_k^{(t,s)}\big), \nabla \mathcal{L}^{(t)}(\theta^{(t)})\big \rangle\\
\geq & -\sum_{k} \sum_{s} \big\|\operatorname{Cov}_t\big(w_k^{(t)}, g_k^{(t,s)}\big)\big\|_F \big\|\nabla \mathcal{L}^{(t)}(\theta^{(t)})\big\|_F
\geq -KS\delta_t \big\|\nabla \mathcal{L}^{(t)}(\theta^{(t)})\big\|_F
\end{align*}
Combing the above facts,
\begin{align}
    \mathbb{E}\big[\big\langle U^{(t)} - \theta^{(t)}, \nabla \mathcal{L}^{(t)}(\theta^{(t)})\big\rangle_F\big]
\leq -\eta \mathbb{E}\bigg[ \sum_{k=1}^K w_k^{(t)} \sum_{s=0}^{S-1}
    \big\langle \nabla \mathcal{L}^{(t)}(\theta^{(t)}), \nabla \mathcal{L}_k^{(t)}(\theta_k^{(t,s)})\big\rangle_F\bigg]\notag\\
+ \eta KS\delta_t \mathbb{E}\big[\big\|\nabla \mathcal{L}^{(t)}(\theta^{(t)})\big\|_F\big].
    \label{eq:inner_decompose}
\end{align}
We split the gradient as $\nabla \mathcal{L}_k^{(t)}(\theta_k^{(t,s)}) = \nabla \mathcal{L}_k^{(t)}(\theta^{(t)}) + \big(\nabla \mathcal{L}_k^{(t)}(\theta_k^{(t,s)}) - \nabla \mathcal{L}_k^{(t)}(\theta^{(t)})\big)$. For the first part, summing over $k$ with weights $w_k^{(t)}$,
\begin{equation}
    \sum_{k=1}^K w_k^{(t)} \big\langle \nabla \mathcal{L}^{(t)}(\theta^{(t)}), \nabla \mathcal{L}_k^{(t)}(\theta^{(t)}) \big\rangle_F
    = \big\| \nabla \mathcal{L}^{(t)}(\theta^{(t)}) \big\|_F^2.
    \label{eq:first_part}
\end{equation}
Summing $s = 0, \ldots, S-1$, this contributes $S\|\nabla \mathcal{L}^{(t)}(\theta^{(t)})\|_F^2$. For the second (drift) part, by Young's inequality with parameter $\nu = 1$ and the $L$-smoothness of $\mathcal{L}_k^{(t)}$,
\begin{equation}
    \big\langle \nabla \mathcal{L}^{(t)}(\theta^{(t)}), \nabla \mathcal{L}_k^{(t)}(\theta_k^{(t,s)}) - \nabla \mathcal{L}_k^{(t)}(\theta^{(t)}) \big\rangle_F
    \ge -\tfrac{1}{2}\big\| \nabla \mathcal{L}^{(t)}(\theta^{(t)})\big\|_F^2 - \tfrac{ L^2}{2} \big\|\theta_k^{(t,s)} - \theta^{(t)}\big\|_F^2.
    \label{eq:second_part}
\end{equation}
Since $w_k^{(t)} \leq 1$ and $\sum_{k=1}^K w_k^{(t)} =1$, we know
\begin{align*}
\sum_{k=1}^K w_k^{(t)} \big\|\theta_k^{(t,s)} - \theta^{(t)}\big\|_F^2
\leq \sum_{k=1}^K \big\|\theta_k^{(t,s)} - \theta^{(t)}\big\|_F^2, \quad
\sum_{k=1}^K w_k^{(t)} \big\| \nabla \mathcal{L}^{(t)}(\theta^{(t)})\big\|_F^2 = \big\| \nabla \mathcal{L}^{(t)}(\theta^{(t)})\big\|_F^2.
\end{align*}
Multiplying by $w_k^{(t)}$ and summing over $k$, then summing over $s$ to both sides of Eq.~\eqref{eq:second_part}, we get
\begin{multline*}
    \sum_{k=1}^K w_k^{(t)}\!\sum_{s=0}^{S-1}\!\big\langle \nabla \mathcal{L}^{(t)}(\theta^{(t)}), \nabla \mathcal{L}_k^{(t)}(\theta_k^{(t,s)}) - \nabla \mathcal{L}_k^{(t)}(\theta^{(t)})\big\rangle_F \\
    \ge  -\tfrac{S}{2}\big\|\nabla \mathcal{L}^{(t)}(\theta^{(t)})\big\|_F^2 - \tfrac{L^2}{2}\!\sum_{k=1}^K\!\sum_{s=0}^{S-1}\!\big\|\theta_k^{(t,s)} - \theta^{(t)}\big\|_F^2.
\end{multline*}
Taking expectation and using $\mathbb{E}\|\theta_k^{(t,s)} - \theta^{(t)}\|_F^2 \le \eta^2 s^2 G^2$ (per-client analogue of Eq.~\eqref{eq:U_W_norm}), together with $\sum_{s=0}^{S-1} s^2 \le S^3/3$, we get
\begin{equation}
    \mathbb{E}\bigg[\sum_{k=1}^K\sum_{s=0}^{S-1}\big\|\theta_k^{(t,s)} - \theta^{(t)}\big\|_F^2\bigg]
    \le \frac{K\eta^2 S^3 G^2}{3}.
\end{equation}
Then,
\begin{multline*}
    \mathbb{E}\Big[\sum_{k=1}^K w_k^{(t)}\!\sum_{s=0}^{S-1}\!\big\langle \nabla \mathcal{L}^{(t)}(\theta^{(t)}), \nabla \mathcal{L}_k^{(t)}(\theta_k^{(t,s)}) - \nabla \mathcal{L}_k^{(t)}(\theta^{(t)})\big\rangle_F\Big] \\
    \ge  -\tfrac{S}{2}\mathbb{E}\big[\big\|\nabla \mathcal{L}^{(t)}(\theta^{(t)})\big\|_F^2\big] - \frac{K\eta^2L^2 S^3 G^2}{6}.
\end{multline*}
Combining with Eq.~\eqref{eq:first_part} and substituting back into Eq.~\eqref{eq:inner_decompose}, we get
\begin{align}
    & \mathbb{E}\big[\big\langle U^{(t)} - \theta^{(t)}, \nabla \mathcal{L}^{(t)}(\theta^{(t)})\big\rangle_F\big]\notag\\
    \le & -\frac{\eta S}{2} \mathbb{E}\big[\big\|\nabla \mathcal{L}^{(t)}(\theta^{(t)})\big\|_F^2\big] + \frac{K \eta^3 L^2 S^3 G^2}{6} + \eta KS\delta_t \mathbb{E}\big[\big\|\nabla \mathcal{L}^{(t)}(\theta^{(t)})\big\|_F\big]\notag\\
    \leq & -\frac{\eta S}{4} \mathbb{E}\big[\big\|\nabla \mathcal{L}^{(t)}(\theta^{(t)})\big\|_F^2\big] + \frac{K \eta^3 L^2 S^3 G^2}{6},
    \label{eq:U_inner}
\end{align}
where we assume $\delta_t \leq \mathbb{E}\big[\big\|\nabla \mathcal{L}^{(t)}(\theta^{(t)})\big\|_F\big]/(4K)$ to derive the last step.
Plugging Eqs.~(\ref{eq:U_W_norm}) and (\ref{eq:U_inner}) into Eq.~(\ref{eq:smooth_U}),
\begin{equation}
    \mathbb{E}\big[\mathcal{L}^{(t)}(U^{(t)})\big]
    \le \mathbb{E}\big[\mathcal{L}^{(t)}(\theta^{(t)})\big] - \frac{\eta S}{4} \mathbb{E}\big\|\nabla \mathcal{L}^{(t)}(\theta^{(t)})\big\|_F^2 + \frac{K \eta^3 L^2 S^3 G^2}{6} + \frac{K L \eta^2 S^2 G^2}{2}.
    \label{eq:U_decrease}
\end{equation}

By the smoothness of $\mathcal{L}^{(t)}$ at $U^{(t)}$,
\begin{equation}
    \mathcal{L}^{(t)}(\theta^{(t+1)}) \le \mathcal{L}^{(t)}(U^{(t)}) + \big\langle \theta^{(t+1)} - U^{(t)}, \nabla \mathcal{L}^{(t)}(U^{(t)})\big\rangle_F + \frac{L}{2}\big\|\theta^{(t+1)} - U^{(t)}\big\|_F^2.
    \label{eq:smooth_V}
\end{equation}

To continue the proof, we first justify Assumption~\ref{assumption:bounded_compression} for our efficiency-based module selection. By construction (Section~\ref{sec:compression}), the operator $\mathcal{C}_k$ drops modules in increasing order of efficiency $E_m = \|\Delta_m\|_F^2/\|\Delta_m\|_0$ until only a fraction $B_k$ of the changed parameters is retained. Let $\bar{E} := \|\Delta\|_F^2/\|\Delta\|_0 = \sum_m E_m \|\Delta_m\|_0 / \sum_m \|\Delta_m\|_0$ denote the overall efficiency averaged over all modules with weights $\|\Delta_m\|_0$. Since dropped modules carry the smallest $E_m$ values, their weighted average $E_m$ (with the same weights $\|\Delta_m\|_0$) is no larger than $\bar{E}$, so the total dropped energy satisfies
\begin{equation}
    \big\|\mathcal{C}_k(\Delta) - \Delta\big\|_F^2 = \sum_{m \in \text{dropped}} E_m \cdot \|\Delta_m\|_0 \le \bar{E} \cdot (1-B_k)\|\Delta\|_0 = (1-B_k)\|\Delta\|_F^2,
    \label{eq:asc_just}
\end{equation}
which validates Assumption~\ref{assumption:bounded_compression}. Now, since $\theta^{(t+1)} - U^{(t)} = \sum_k w_k^{(t)} (\tilde{\Delta}_k^{(t)} - \Delta_k^{(t)})$, applying Jensen's inequality together with $w_k^{(t)} \le 1$, we get
\begin{align*}
\mathbb{E}\big\|\theta^{(t+1)} - U^{(t)}\big\|_F^2
= & \mathbb{E}\Big[\Big\|\sum_{k=1}^K w_k^{(t)}(\tilde{\Delta}_k^{(t)} - \Delta_k^{(t)})\Big\|_F^2\Big]
\leq \mathbb{E}\Big[\sum_{k=1}^K w_k^{(t)}\|\tilde{\Delta}_k^{(t)} - \Delta_k^{(t)}\|_F^2\Big]\\
\leq & \mathbb{E}\Big[\sum_{k=1}^K \|\tilde{\Delta}_k^{(t)} - \Delta_k^{(t)}\|_F^2\Big]
\leq \mathbb{E}\Big[\sum_{k=1}^K (1-B_k^{(t)})\|\Delta_k^{(t)}\|_F^2\Big]
\end{align*}
where we use Eq.~\eqref{eq:asc_just} in the last step. Applying $\mathbb{E}[X Y]=\mathbb{E}[X] \mathbb{E}[Y]+\operatorname{Cov}(X, Y)$, this further imples
\begin{equation}
\mathbb{E}\big\|\theta^{(t+1)} - U^{(t)}\big\|_F^2
=\sum_{k=1}^K \mathbb{E}[(1-B_k^{(t)})] \mathbb{E}[\|\Delta_k^{(t)}\|_F^2] - \sum_{k=1}^K \operatorname{Cov}(B_k^{(t)},\|\Delta_k^{(t)}\|_F^2)
\le \frac{1}{2}K(1-\bar{B}^{(t)}) \eta^2 S^2 G^2,
    \label{eq:V_U_norm}
\end{equation}
where $\bar{B}^{(t)} := \tfrac{1}{K}\sum_{k=1}^K B_k^{(t)}$ is the per-client average budget ratio at round $t$. Here we used $\mathbb{E}\big\|\Delta_k^{(t)}\big\|_F^2 \le \eta^2 S^2 G^2$ (analogue to Eq.~\eqref{eq:U_W_norm}) together with the assumption that $\operatorname{Cov}(B_k^{(t)},\|\Delta_k^{(t)}\|_F^2) \geq -\frac{1}{2}(1-\bar{B}^{(t)}) \eta^2 S^2 G^2$. For the inner-product term in Eq.~(\ref{eq:smooth_V}), we first note
\begin{equation}
    \big\|\nabla \mathcal{L}_k^{(t)}(\theta)\big\|_F^2 = \big\|\mathbb{E}\nabla \mathcal{L}_k^{(t)}(\theta;\xi_{k,s}^{(t)})\big\|_F^2 \le \mathbb{E}\big\|\nabla \mathcal{L}_k^{(t)}(\theta;\xi_{k,s}^{(t)})\big\|_F^2 \le G^2
    \label{eq:bounded_full_grad}
\end{equation}
Hence by Jensen's inequality, 
\begin{equation}
    \big\|\nabla \mathcal{L}^{(t)}(U^{(t)})\big\|_F^2 = \Big\|\sum_{k=1}^K w_k^{(t)}\nabla \mathcal{L}_k^{(t)}(U^{(t)})\Big\|_F^2 \le \sum_{k=1}^K w_k^{(t)}\big\|\nabla \mathcal{L}_k^{(t)}(U^{(t)})\big\|_F^2 \le G^2.
    \label{eq:bounded_global_grad}
\end{equation}
Now, by Young's inequality with parameter $\eta_0 > 0$ chosen below,
\begin{align}
    & \mathbb{E}\big[\big\langle \theta^{(t+1)} - U^{(t)}, \nabla \mathcal{L}^{(t)}(U^{(t)})\big\rangle_F\big] \notag\\
    \le \ & \frac{\eta_0 G^2}{2} + \frac{1}{2\eta_0} \mathbb{E}\big\|\theta^{(t+1)} - U^{(t)}\big\|_F^2 \notag\\
    \le \ & \frac{\eta_0 G^2}{2} + \frac{K(1-\bar{B}^{(t)}) \eta^2 S^2 G^2}{2\eta_0}.
    \label{eq:V_inner}
\end{align}
We choose $\eta_0 = \sqrt{K(1-\bar{B}^{(t)})} \eta S$ (when $\bar{B}^{(t)} < 1$), which is the AM--GM-optimal value that balances the two terms, making both equal to $\sqrt{K(1-\bar{B}^{(t)})}\eta S G^2/2$. (When $\bar{B}^{(t)} = 1$, we have $\theta^{(t+1)} = U^{(t)}$ a.s.\ and Eq.~(\ref{eq:V_inner}) trivially holds.) Plugging Eqs.~(\ref{eq:V_U_norm}) and (\ref{eq:V_inner}) into Eq.~(\ref{eq:smooth_V}),
\begin{equation}
    \mathbb{E}\big[\mathcal{L}^{(t)}(\theta^{(t+1)})\big]
    \le \mathbb{E}\big[\mathcal{L}^{(t)}(U^{(t)})\big] + \sqrt{K(1-\bar{B}^{(t)})}  \eta S G^2 + \frac{L K(1-\bar{B}^{(t)}) \eta^2 S^2 G^2}{4}.
    \label{eq:V_decrease}
\end{equation}

Combining Eq.~(\ref{eq:U_decrease}) and Eq.~(\ref{eq:V_decrease}), and rearranging,
\begin{equation}
\begin{aligned}
    \frac{\eta S}{4} \mathbb{E}\big\|\nabla \mathcal{L}^{(t)}(\theta^{(t)})\big\|_F^2
    \le \ & \mathbb{E}\big[\mathcal{L}^{(t)}(\theta^{(t)})\big] - \mathbb{E}\big[\mathcal{L}^{(t)}(\theta^{(t+1)})\big] \\
    & + \underbrace{\sqrt{K(1-\bar{B}^{(t)})}  \eta S G^2}_{\text{leading-order residual}} \\
    & + \underbrace{\frac{K \eta^3 L^2 S^3 G^2}{6} + \frac{K L \eta^2 S^2 G^2 (2-\bar{B}^{(t)})}{4}}_{=: \Phi(\eta)\text{(higher-order residuals)}}.
    \label{eq:per_round}
\end{aligned}
\end{equation}
There exists a constant $M = M(L, G, S, K, \eta_{\max}) > 0$ such that $\Phi(\eta) \le (M\eta) \cdot (\eta S/4) = M\eta^2 S/4$ for all $\eta \in (0, \eta_{\max}]$, so $\Phi(\eta)/(\eta S/4) \le M\eta$. Concretely, since $\Phi(\eta)/(\eta^2 S/4) = \tfrac{2K \eta L^2 S^2 G^2}{3} + KLSG^2(2-\bar B^{(t)})$, we can take $M = \tfrac{2K \eta_{\max} L^2 S^2 G^2}{3} + 2KLSG^2$. 

The descent inequality~(\ref{eq:per_round}) involves $\mathcal{L}^{(t)}$ on both sides. To telescope across rounds, we connect $\mathcal{L}^{(t)}(\theta^{(t+1)})$ on the LHS of round $t$ to $\mathcal{L}^{(t+1)}(\theta^{(t+1)})$ which appears on the RHS of round $t+1$. By the definition of $\Delta^{(t)}_{\mathcal{L}}$ in Assumption~\ref{assumption:drift},
\begin{equation}
    \mathcal{L}^{(t)}(\theta^{(t+1)}) \ge \mathcal{L}^{(t+1)}(\theta^{(t+1)}) - \Delta^{(t)}_{\mathcal{L}}.
    \label{eq:drift_link}
\end{equation}
Substituting Eq.~(\ref{eq:drift_link}) into Eq.~(\ref{eq:per_round}), 
\begin{equation}
    \frac{\eta S}{4} \mathbb{E}\big\|\nabla \mathcal{L}^{(t)}(\theta^{(t)})\big\|_F^2
    \le \mathbb{E}\big[\mathcal{L}^{(t)}(\theta^{(t)})\big] - \mathbb{E}\big[\mathcal{L}^{(t+1)}(\theta^{(t+1)})\big] + \mathbb{E}[\Delta^{(t)}_{\mathcal{L}}]
    + \sqrt{K(1-\bar{B}^{(t)})} \eta S G^2 + \Phi(\eta).
    \label{eq:per_round_drift}
\end{equation}

Before telescoping, we fix the boundary term at $t=T$. Define $\mathcal{L}^{(T)}$ using any uniformly bounded extension of the round-$T$ statistics (e.g., reusing the round-$(T{-}1)$ weights and calibration) so that $\mathcal{L}^{(T)}(\theta^{(T)}) \ge \inf_{0\le t\le T,\,\theta}\mathcal{L}^{(t)}(\theta)$; the specific choice of extension does not affect the bound, since the only round-$T$ quantity entering the telescoped sum is $\rho^{(T-1)}$, which is absorbed by Assumption~\ref{assumption:drift}. With this convention, $D \ge \mathcal{L}^{(0)}(\theta^{(0)}) - \inf_{0\le t\le T,\,\theta}\mathcal{L}^{(t)}(\theta) \ge \mathbb{E}[\mathcal{L}^{(0)}(\theta^{(0)})] - \mathbb{E}[\mathcal{L}^{(T)}(\theta^{(T)})]$.

Summing Eq.~(\ref{eq:per_round_drift}) over $t = 0, 1, \ldots, T-1$, and using the bound on $D$ above together with $\sum_t \mathbb{E}[\Delta^{(t)}_{\mathcal{L}}] \le RT$ from Assumption~\ref{assumption:drift} and Jensen's inequality $\frac{1}{T}\sum_t \sqrt{1-\bar{B}^{(t)}} \le \sqrt{1-\bar{B}}$ over rounds (where $\bar{B} = \frac{1}{T}\sum_t \bar{B}^{(t)}$ is the budget averaged over both clients and rounds),
\begin{equation}
    \frac{\eta S}{4}\sum_{t=0}^{T-1}\mathbb{E}\big\|\nabla \mathcal{L}^{(t)}(\theta^{(t)})\big\|_F^2
    \le D + RT + T \cdot \bigg[\sqrt{K(1-\bar{B})} \eta S G^2 + \Phi(\eta)\bigg].
    \label{eq:tele_sum}
\end{equation}
Dividing both sides by $\eta S T/4$ (and using $\Phi(\eta)/(\eta S/4) \le M\eta$),
\begin{equation}
    \frac{1}{T}\sum_{t=0}^{T-1}\mathbb{E}\big\|\nabla \mathcal{L}^{(t)}(\theta^{(t)})\big\|_F^2
    \le \frac{4D}{\eta S T} + \frac{4R}{\eta S} + 4\sqrt{K(1-\bar{B})} G^2 + M \eta.
    \label{eq:tele_avg}
\end{equation}
Setting $\eta = \sqrt{2(D/T+R)/(MS)}$ to optimally balance the $\eta$-dependent terms (i.e., $\tfrac{2(D/T+R)}{\eta S}$ against $M\eta$) via AM--GM,
\begin{equation}
    \frac{1}{T}\sum_{t=0}^{T-1}\mathbb{E}\big\|\nabla \mathcal{L}^{(t)}(\theta^{(t)})\big\|_F^2
    \le 2\sqrt{\frac{2 M (D/T+R)}{S}} + 4\sqrt{K(1-\bar{B})} G^2.
    \label{eq:final_pre}
\end{equation}
Rescaling the constant $M$ to absorb the factor $S/2$ (i.e., $M' = 2M/S$, and renaming $M' \to M$), and using $\sqrt{a+b} \le \sqrt{a}+\sqrt{b}$ for $a,b \ge 0$, we obtain the form stated in Eq.~(\ref{eq:thm1_main}):
\begin{equation}
    \frac{1}{T}\sum_{t=0}^{T-1}\mathbb{E}\big\|\nabla \mathcal{L}^{(t)}(\theta^{(t)})\big\|_F^2
    \le 2\sqrt{\frac{D M}{T}} + 2\sqrt{R M} + 4\sqrt{K(1-\bar{B})} G^2.
    \label{eq:final}
\end{equation}

\end{proof}

\section{Experimental Details}
\label{app_sec:full_exp_details}

\paragraph{Datasets.}
For mathematical reasoning tasks, we use the MATH dataset~\cite{hendrycks2021measuring}, which contains approximately 7{,}500 competition-level math problems spanning 5 difficulty levels. To construct a naturally non-IID federated partition, we split the training set across $K=5$ clients by difficulty level, so that each client specializes in problems of a single difficulty. For evaluation, we report accuracy on five widely-used math reasoning benchmarks: MATH-500~\cite{hendrycks2021measuring,lightman2023let}, GSM8K~\cite{cobbe2021training}, AIME 2024, AIME 2025, and AMC 2023. \pengxin{\pengxin{For coding tasks, we train on the MBPP dataset~\cite{austin2021program} ("train" + "validation") and evaluate on the held-out MBPP "test" split and the HumanEval dataset~\cite{chen2021evaluating}. To construct a non-IID federation setting, we partition the MBPP dataset by the reference-solution length, bucket it into 5 quantile bins, and partition the data across clients with a Dirichlet distribution ($\alpha=0.5$) over these buckets.}}

\paragraph{Baselines.}
We compare Fed-GRPO against several baselines that represent different points in the centralized--federated spectrum: i) \textbf{Base Model}: base model without any training, serving as base performance; ii) \textbf{Centralized}: standard GRPO training on the entire MATH dataset, serving as an upper bound for federated methods; iii) \textbf{Local-Only}: each client independently runs GRPO on its local data without any communication or aggregation; iv) \textbf{FedAvg-SW} / \textbf{FedAvg-UW}: standard FedAvg~\cite{mcmahan2017communication} with GRPO as the local training algorithm. SW uses sample-weighted aggregation ($w_k = n_k / N$), while UW uses uniform weights ($w_k = 1/K$); v) \textbf{SparsyFed}~\cite{guastella2025sparsyfed} and \textbf{SparseLoCo}~\cite{sarfi2025communication}: communication-effecient federated learning methods that apply fixed-ratio top-$k$ sparsification to the model updates.
\pengxin{vi) \textbf{FGRPO}~\cite{chen2026fgrpo}: incorporates an adaptive aggregation mechanism based on relative performance gain, characterizing each client’s improvement relative to its personalized historical baseline to dynamically prioritize effective learning trajectories regardless of local task difficulty.}

\paragraph{Training details.}
We use Qwen2.5-3B-Instruct~\cite{qwen2025qwen25} \pengxin{and Qwen3-4B-Instruct~\cite{yang2025qwen3}} as the base model for all methods. We use the AdamW~\cite{loshchilov2018decoupled} optimizer with learning rate $3 \times 10^{-6}$, $\beta_1=0.9$, $\beta_2=0.999$, weight decay $0$, a constant learning-rate schedule with warmup ratio $0.1$, and gradient clipping at $1.0$ for GRPO training. For rollout, the prompt batch size is 64 and we sample 8 responses for each prompt at sampling temperature $1.0$ and top-$p=1.0$. The clip ratio is set to $\epsilon=0.2$, KL coefficient $\beta_{\text{KL}}=0$, and we use BF16 mixed precision. The maximum prompt and completion lengths are capped at $1024$ and $2048$ tokens, respectively. The reward function is a binary accuracy reward that checks whether the extracted answer matches the ground truth \pengxin{or whether the code is right or not}. All federated methods run for $T=12$ communication rounds with $S=5$ local steps per round. Centralized and local-only baselines are trained for up to $310$ steps (matching the total number of gradient updates in the federated setting). For Fed-GRPO, we set the global reward calibration strength $\lambda = 0.5$; for adaptive sparse communication, we vary the budget $B$ to sweep the compression--accuracy trade-off, with minimum budget $B_{\min}=0.01$. All experiments are implemented on top of TRL~\cite{vonwerra2020trl} library and run on $8$ NVIDIA H800 GPUs.

\paragraph{Evaluation protocol.}
We save a model checkpoint at the end of each communication round for federated methods, and every $10$ training steps for centralized and local-only baselines, then evaluate each checkpoint on all five benchmarks. Following common practice in GRPO literature~\cite{yu2025dapo, guo2025deepseek}, we report the best accuracy achieved across checkpoints. For inference, we sample with temperature $1.0$, top-$p=0.7$, and a maximum generation length of $4096$ tokens; we generate $32$ samples per prompt and report avg@$32$ (the mean accuracy over the $32$ samples) for results stability.

{\color{black}
\section{More Experimental Results}
\label{sec:more_exp_results}

\subsection{Results on Different Data Heterogeneity}

We further evaluate the performance of Fed-GRPO under different data heterogeneity settings. Specifically, we partition the MATH dataset by "subject" (Algebra, Geometry, Number Theory, Combinatorics) using a Dirichlet distribution with $\alpha=0.5$ to create a non-IID federation across $K=20$ clients. The results are summarized in Table~\ref{tab:main_results_dirichlet_subject}. We observe that Fed-GRPO consistently outperforms all baselines across all benchmarks, demonstrating its robustness and effectiveness in handling heterogeneous data distributions.

\begin{table}[h]
\caption{Results on five math reasoning benchmarks using Qwen2.5-3B-Instruct as the base model under dirichlet partitioning ($\alpha$=0.5) by "subject" with K=20 clients.}
\label{tab:main_results_dirichlet_subject}
\vspace{-0.3cm}
\begin{center}
\begin{tabular}{lccccc|c}
\toprule
\textbf{Method} & \textbf{GSM8K} & \textbf{MATH-500} & \textbf{AIME24} & \textbf{AIME25} & \textbf{AMC} & \textbf{Avg.} \\

\midrule
Qwen2.5-3B-Instruct & 59.47 & 46.54 & 2.50 & 1.46 & 27.11 & 27.42 \\
\midrule
Centralized & 84.75 & 63.42 & 7.50 & 3.44 & 41.95 & 40.21 \\
\midrule
Local-Only &  79.60 & 53.09 & 4.43 & 2.26 & 33.57 & 34.59 \\
FedAvg-SW & 81.32 & 56.04 & 5.04 & 2.72 & 35.12 & 36.05 \\
FedAvg-UW & 80.79 & 56.23 & 4.97 & 2.66 & 34.86 & 35.90 \\
SparsyFed & 80.96 & 56.02 & 5.02 & 2.53 & 35.12 & 35.93 \\
SparseLoCo &  81.02 & 55.96 & 4.89 & 2.55 & 34.97 & 35.22\\
FGRPO & 81.43 & 56.77 & 5.48 & 2.76 & 38.10 & 36.91 \\
\cmidrule{1-7}
Fed-GRPO (Ours) & \textbf{82.02} & \textbf{58.77} & \textbf{6.18} & \textbf{3.02} & \textbf{39.11} & \textbf{38.02}
 \\
\bottomrule
\end{tabular}
\end{center}
\end{table}

\subsection{Dynamics of Signal-Weighted Aggregation}
\label{sec:deep_swa}

\begin{wrapfigure}{r}{0.52\textwidth}
    \centering
    \begin{subfigure}{0.25\textwidth}
        \centering
        \includegraphics[width=\textwidth]{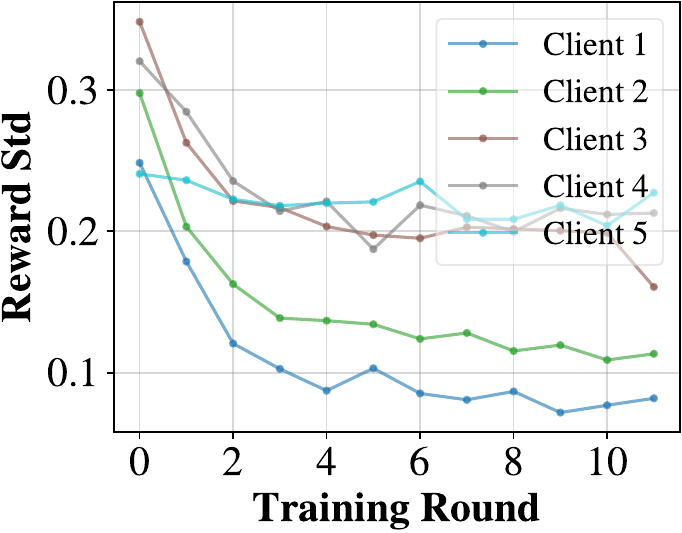}
        \caption{}
        \label{fig:swa_sigma}
    \end{subfigure}
    \begin{subfigure}{0.25\textwidth}
        \centering
        \includegraphics[width=\textwidth]{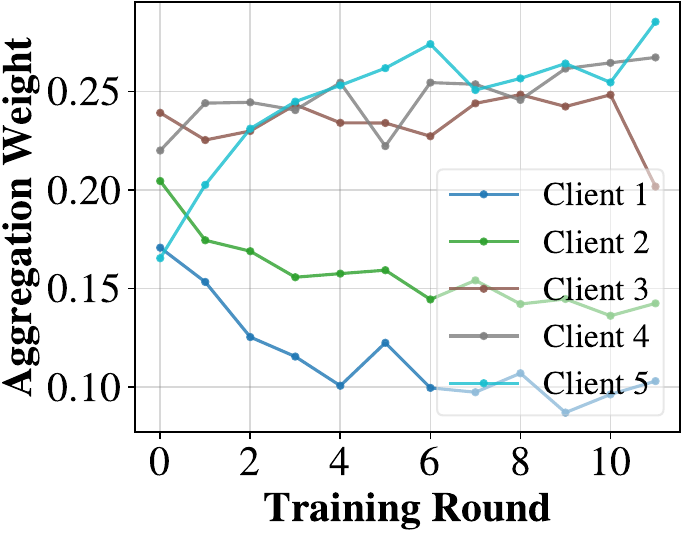}
        \caption{}
        \label{fig:swa_weights}
    \end{subfigure}
    \caption{Dynamics of (a) per-client reward standard deviation $\sigma_k$ and (b) aggregation weights $w_k$. The aggregation weights shift dynamically across rounds, concentrating on clients with stronger learning signals (higher $\sigma_k$) as training progresses.}
    \label{fig:swa_dynamics}
\end{wrapfigure}

To understand how signal-weighted aggregation adapts during training, we track two quantities across 12 communication rounds on the MATH dataset using Qwen2.5-3B-Instruct as the base model: the per-client reward standard deviation $\sigma_k$ (Figure~\ref{fig:swa_dynamics}(a)) and the resulting aggregation weights $w_k$ (Figure~\ref{fig:swa_dynamics}(b)).
As shown in Figure~(\ref{fig:swa_sigma}), $\sigma_k$ evolves differently for each client, reflecting their varying learning signals as training progresses. Some clients experience an increase in $\sigma_k$, indicating stronger learning signals, while others see a decrease in $\sigma_k$ as they master their local data. Figure~(\ref{fig:swa_weights}) shows that the aggregation weights $w_k$ adapt accordingly, shifting more weight toward clients with higher $\sigma_k$. This dynamic weighting allows Fed-GRPO to prioritize updates from clients that are currently contributing more to global improvement, effectively mitigating the aggregation mismatch issue in federated GRPO training. The fact that the weights change across rounds also suggests that the method is responsive to the evolving training dynamics, rather than relying on static heuristics. Overall, this analysis confirms that signal-weighted aggregation successfully captures and leverages the varying learning signals across clients to improve the quality of the global update. 

}

\end{document}